\documentclass[11pt, a4paper]{article}

\usepackage{amssymb}
\usepackage{amsmath}
\usepackage{amsthm}
\usepackage{mathtools}

\newcommand{\R}{\mathbb{R}}
\newcommand{\N}{\mathbb{N}}

\def\Lip{\text{\rm Lip}}

\newcommand{\F}{{\mathcal{F}}}
\newcommand{\h}{{\mathcal{H}}}
\newcommand{\s}{{\mathcal{S}}}

\newtheorem{theorem}{Theorem}[section]
\newtheorem*{theorem*}{Theorem}
\newtheorem{lemma}[theorem]{Lemma}
\newtheorem*{lemma*}{Lemma}

\newtheorem*{corollary*}{Corollary}
\newtheorem{proposition}[theorem]{Proposition}
\newtheorem{definition}[theorem]{Definition}
\newtheorem*{definition*}{Definition}

\newtheorem{remark}[theorem]{Remark}
\newtheorem*{remark*}{Remark}

\numberwithin{equation}{section}

\providecommand{\keywords}[1]
{
  \small	
  \textbf{Keywords: } #1
}

\title{
    Universal approximation property of ODENet and ResNet\\ with a single activation function
}

\author{Masato Kimura\thanks{Faculty of Mathematics and Physics, Kanazawa University, Japan (mkimura@se.kanazawa-u.ac.jp)}\and Kazunori Matsui\thanks{Department of Logistics and Information Engineering, Tokyo University of Marine Science and Technology, Japan (kmat002@kaiyodai.ac.jp
)}\and Yosuke Mizuno\thanks{Division of Mathematical and Physical Sciences, Kanazawa University, Japan (marble.89422m@stu.kanazawa-u.ac.jp)}}
\date{}

\begin{document}
    
\maketitle

\begin{abstract}
    We study a universal approximation property of ODENet and ResNet. The ODENet is a map from an initial value to the final value of an ODE system in a finite interval. It is considered a mathematical model of a ResNet-type deep learning system. We consider dynamical systems with vector fields given by a single composition of the activation function and an affine mapping, which is the most common choice of the ODENet or ResNet vector field in actual machine learning systems. We show that such an ODENet and ResNet with a restricted vector field can uniformly approximate ODENet with a general vector field.
\end{abstract}


\section{Introduction}
Neural networks have significantly impacted computer vision, natural language processing, and learning methods \cite{Schmidhuber15}. Historically, these networks have been seen as models inspired by the human brain and eye \cite{FM82,MP43}. A neural network, with the input layer as layer $0$, the output layer as layer $L$, and the dimensions of each layer as $N\in\N$, is represented by the following equations:
\begin{equation*}
    \left\{\begin{aligned}
        x^{(l+1)} &= f^{(l)}(x^{(l)}) & (l=0,1,\ldots,L-1), \\
        x^{(0)} &= \xi \in D, &
    \end{aligned}\right.
\end{equation*}
where $f^{(l)}:\mathbb{R}^N\to\mathbb{R}^N$ and $D\subset\R^N$ is a bounded closed subset.
The network's input and output are $\xi$ and $x^{(L)}$, respectively.
Neural networks with a large number of layers are called deep networks and perform better than single-hidden-layer neural networks for large-scale and high-dimensional challenges. Notable models like AlexNet \cite{AlexNet} and GoogLeNet \cite{GoogLeNet} contain many layers.

However, if the depth is overly increased, the accuracy might stagnate or degrade \cite{HS15}. Additionally, gradient vanishing or exploding issues may arise in deeper networks, making these networks harder to train \cite{BSF94,GB10}. To address these issues, the authors in \cite{HZRS16} recommended using residual learning to facilitate the training of considerably deeper networks than previously used. Such a network is called a residual network or ResNet, which has the form: 
\begin{equation}\label{eq:intro-resnet}
    \left\{\begin{aligned}
        x^{(l+1)} &= x^{(l)} + f^{(l)}(x^{(l)}) & (l=0,1,\ldots,L-1), \\
        x^{(0)} &= \xi \in D. &
    \end{aligned}\right.
\end{equation}
ResNet can be viewed as an explicit discretization of a system of
ordinary differential equations \cite{CRBD18}.
Choosing a small $h>0$ and denoting $f^{(\ell)}/h$ again by $f^{(\ell)}$
in \eqref{eq:intro-resnet}, we obtain
\begin{equation}\label{eq:intro-euler-method}
    \left\{\begin{aligned}
        x^{(l+1)} &= x^{(l)}+hf^{(l)}(x^{(l)}) & (l=0,1,\ldots,L-1), \\
        x^{(0)} &= \xi \in D, &
    \end{aligned}\right.
\end{equation}
which is the explicit Euler scheme for the following
initial value problem of a system of ODEs.
Putting $x(t):=x^{(l)}$ and $f(x,t):=f^{(l)}(x)$, where $t=hl$, $T=hL$ and $f:[0,T]\times\mathbb{R}^N\rightarrow\mathbb{R}^N$, and taking the limit of \eqref{eq:intro-euler-method} as $h$ approaches zero, we formally obtain the following initial value problem of a system of ODEs:
\begin{align}\label{ode}
    \left\{\begin{array}{ll}
        x'(t)=f(x(t),t)\quad(0\le t\le T),\\
        x(0)=\xi\in D.
    \end{array}\right.
\end{align}
In the context of machine learning,
it is called a neural ordinary differential equation (NODE) \cite{CRBD18}.
We call the function $D\ni\xi\mapsto x(T)\in\mathbb{R}^N$ an ODENet with a vector field $f$. The utilization of NODEs is significant in invertible neural network architectures for constructing more expressive generative models and continuous normalizing flows, facilitating advanced generative modeling and density estimation tasks \cite{FJNO20, GCBSD19, PNRML21} (cf. \cite{HHY21}). For a comprehensive overview, we refer to \cite{Kidger21}. We restrict the function $f$ to the following form:
\begin{equation}\label{eq:intro-function-form}
    f(x,t)=\alpha(t)\odot\mbox{\boldmath $\sigma$}(\beta(t)x+\gamma(t)),
\end{equation}
where $\alpha:[0,T]\rightarrow\mathbb{R}^N$ is a weight vector, $\beta:[0,T]\rightarrow\mathbb{R}^{N\times N}$ is a weight matrix, and $\gamma:[0,T]\rightarrow\mathbb{R}^N$ is a bias vector. The operator $\odot$ denotes the Hadamard product (element-wise product) of two vectors defined by \eqref{eq:Hadamard-product}. The function $\mbox{\boldmath $\sigma$}:\mathbb{R}^N\to\mathbb{R}^N$ is defined by 
\begin{equation}\label{eq:activation-function}
    \mbox{\boldmath $\sigma$}(x):=\left(\sigma(x_1) ,\sigma(x_2),\cdots, \sigma(x_N)\right)^\top,
\end{equation}
where $\sigma:\mathbb{R}\to\mathbb{R}$ is called an activation function. The activation function is, for example, the sigmoid function $\sigma(x)=(1+e^{-x})^{-1}$, the hyperbolic tangent function $\sigma(x)=\tanh(x)$, and the rectified linear unit (ReLU) function $\sigma(x)=\max\{0,x\}$.

It has been shown in \cite{LLS23, RBZ21} that an ODENet has universal approximation property (UAP) in the $L^p$-norm, and in particular, \cite{LLS23} also proves UAP in the sup-norm in the one-dimensional case. The UAP in the sup-norm is shown in \cite{TG21} under the assumption that the activation function is a solution of a quadratic differential equation (and the target function is monotone and analytic), and in \cite{aizw} for the case where the layer width is $2N$.
Additionally, \cite{TTIIO20} shows that an ODENet with a vector field given by a linear combination of activation functions (Remark \ref{Rem:adv}) has UAP (cf. \cite{EZOP23, ZdAFT23}).

From the general discussion on the explicit Euler method, if ODENet has UAP, then ResNet also has UAP \cite{aizw, LLS23, RBZ21, ZdAFT23}. In \cite{LJ18}, it is shown that a ResNet with the ReLU activation function, where the dimensions of the input layer and output one are $N$ and $1$, respectively, alternating layers of $1$ and $N$ dimensions, has UAP in the $L^1$-norm.

In this paper, we establish the conditions for UAP in the sup-norm for the ODENet with respect to any non-polynomial continuous activation function that satisfies a Lipschitz condition. Here, the layer width is $N$.
Specifically, we prove that if $\sigma$ has UAP and satisfies a Lipschitz condition, then the ODENet with respect to any function satisfying the Lipschitz condition can be approximated by the ODENet with respect to a function of type \eqref{eq:intro-function-form}. Additionally, we also demonstrate that ResNet has UAP. 
The organization of this paper is as follows. Section \ref{sec2} defines the notation and function spaces used in this paper, with particular emphasis on the definitions of ODENet, ResNet, and the universal approximation property. Section \ref{sec3} introduces the main theorems. In Section \ref{sec4}, we prepare the necessary theorems and lemmas for proving the main theorems, and in Section \ref{sec5}, we provide the proofs of the main theorems.

\section{ODENet and universal approximation property}\label{sec2}

For $a=(a_1,a_2,\ldots,a_N)^{\top},b=(b_1,b_2,\ldots,b_N)^{\top}\in\mathbb{R}^N$, the Hadamard product of $a$ and $b$ is defined by
\begin{equation}\label{eq:Hadamard-product}
    a\odot b:=\left(
        a_1b_1 ,
        a_2b_2 ,
        \cdots,
        a_Nb_N\right)^\top\in\mathbb{R}^N.
\end{equation}
For a function $f:\R^N\times[0,T]\rightarrow\R^N$, we define its Lipschitz constant by
\begin{align*}
\Lip (f)\coloneqq\sup\left\{\frac{|f(z,t)-f(\zeta,t)|}{|z-\zeta|}\middle|\;z,\zeta\in\R^N,\;z\neq\zeta,\; t\in[0,T]\right\},
\end{align*}
and, if $\Lip (f)<\infty$, we say that $f$ satisfies a Lipschitz condition with respect to $x$. 

We define a solution to the system of ODEs \eqref{ode} as
\begin{align}\label{sol_ode}
    \left\{\begin{aligned}
        x(t) &= \xi+\int^t_0f(x(s),s)ds\quad(0\le t\le T),\\
        x &\in C^0([0,T];\R^N).
    \end{aligned}\right.
\end{align}
Unless otherwise specified, we assume that
\begin{align}\label{asmp_fLip}
    f:\R^N\times[0,T]\rightarrow\R^N\mbox{ satisfies a Lipschitz condition with respect to }x
\end{align}
and 
\begin{align}\label{asmp_fdiscr}
    f\in L^\infty_{\rm loc}(\R^N\times[0,T]).
\end{align}
Although we do not assume continuity of $f$ with respect to time, there exists a unique solution to \eqref{ode} by Theorem \ref{thm:Existence and Uniqueness of Solutions}.

Let $D\subset\R^N$ be a bounded closed subset. We define ODENet and ResNet as follows.
\begin{definition}[ODENet]\label{dfn1}
For the system of ODEs \eqref{ode}, we call $S_f(t)\xi\coloneqq x(t)\quad (0\le t\le T)$ a ``flow'' associated with the vector field $f$ and call $S_f(T): \xi\mapsto x(T)$ ODENet with respect to $f$.
\end{definition}
\begin{definition}[ResNet]\label{dfn2}
    For a sequence of functions $f_R=\{f^{(l)}_R:\R^N\rightarrow\R^N\}_{l=0}^{L-1}$, we define mappings $S^{res}_{f_R}(l):D\ni\xi\mapsto x^{(l)}\in\R^N$, for $l \in 0, \ldots, L$, by the following recurrence formula
    \begin{equation}\label{resnet_th2}
        \left\{\begin{aligned}
        x^{(l+1)} &= x^{(l)} + f^{(l)}_R(x^{(l)}), & (l=0,1,\ldots,L-1),\\
        x^{(0)} &= \xi \in D,
        \end{aligned}\right.
    \end{equation}
    In particular, we call $S^{res}_{f_R}(L)$ ResNet with respect to $f_R$.
\end{definition}

\begin{definition}[Universal approximation property for the activation function $\sigma$]\label{dfn3}
    For $\sigma: \mathbb{R}\rightarrow\mathbb{R}$, consider the set
    \begin{equation*}
        S:=\left\{
            G:D\to\mathbb{R}\left|\;
            G(\xi) = \sum_{l=1}^L\alpha_l\sigma(\beta_l\cdot\xi+\gamma_l),
            L\in\mathbb{N},
            \alpha_l,\gamma_l\in\mathbb{R},\beta_l\in\mathbb{R}^N\right.
        \right\}.
    \end{equation*}
    Suppose that $S$ is dense in $C^0(D)$. In other words, given $F\in C^0(D)$ and $\epsilon>0$, there exists a function $G\in S$ such that
    \begin{equation*}
        |G(\xi)-F(\xi)|<\epsilon
    \end{equation*}
    for all $\xi\in D$. Then, we say that $\sigma$ has a universal approximation property (UAP) on $D$.
\end{definition}

In this paper, we assume that $\sigma\in C^0(\R)$ satisfying the following condition:
\begin{align}\label{asmp_sigma}\begin{aligned}
    &\sigma\mbox{ has UAP on all bounded closed subsets and 
    satisfies the Lipschitz condition with respect to }x.
\end{aligned}\end{align}
A non-polynomial function has UAP on all bounded closed subsets \cite{LLPS93, SM17} (cf. \cite{cyb}).
For example, the truncated power function $x\mapsto x_{+}^k$, defined by
\[
    x_{+}^k:=\left\{\begin{array}{ll}
        x^k & x>0 \\
        0 & x\le0
    \end{array}\right.\quad (k\in\mathbb{N}),
\]
the ReLU function $x\mapsto x^1_+$, the softplus function $x\mapsto \log(1+e^x)$, the sigmoid function $x\mapsto (1+e^{-x})^{-1}$, and the hyperbolic tangent function $x\mapsto \tanh(x)$ have UAP on all bounded closed subsets.

We introduce the notations used in this work. We define the following sets:
\begin{align*}
    \F&\coloneqq\left\{ f\in C^0(\R^N\times[0,T];\R^N) \left|\; \Lip (f)<\infty\right.\right\},\\
    \h&\coloneqq\left\{h:\R^N\times[0,T]\rightarrow\R^N
    \left| \begin{array}{ll}
        h(x,t)=\alpha(t)\odot\boldsymbol{\sigma}(\beta(t)x+\gamma(t))
        \text{ for all }(x,t)\in\R^N\times[0,T],\\
        \text{where }\alpha,\gamma\in C^0([0,T];\R^N),
        \beta\in C^0([0,T];\R^{N\times N})
    \end{array}\right.\right\}\subset\F,\\
    \s(\F)&\coloneqq\left\{S_f(T)\left|f\in\F\right.\right\}\subset C^0(D;\R^N),\\
    \s(\h)&\coloneqq\left\{S_h(T)\left| h\in\h\right.\right\}\subset\s(\F),\\
    \s^{res}&\coloneqq\left\{S^{res}_{f_R}(L)
    \middle|\; \begin{array}{ll}
        f_R=\{f^{(l)}_R:\R^N \ni x\mapsto \alpha_l\odot\boldsymbol{\sigma}(\beta_l x+\gamma_l)\}_{l=0}^{L-1}, L\in\N, \\
        \{\alpha_{l}\}^{L-1}_{l=0},\{\gamma_{l}\}^{L-1}_{l=0}\subset \R^N, \{\beta_{l}\}^{L-1}_{l=0}\subset \R^{N\times N}
    \end{array}
    \right\}.
\end{align*}
Let $\eta\in C^\infty_0(\R^N)$ satisfy that
\begin{align*}
    \eta(x)\ge 0~(x\in\R^N),\quad
    \eta(x)=0~(x\in\R^N,\; |x|\ge 1),\quad
    \int_{\R^N}\eta(x)dx=1.
\end{align*}
For $\alpha,\gamma\in L^\infty(0,T;\R^N)$, $\beta\in L^\infty(0,T;\R^{N\times N})$, and $\delta>0$, we put 
\begin{align*}\begin{array}{c}
    \tilde{\alpha}(t)=\left\{\begin{array}{ll}
        \alpha(t)\quad(0\le t\le T)\\
        0\quad(t<0,T<t)
    \end{array}\right.,\quad
    \tilde{\beta}(t)=\left\{\begin{array}{ll}
        \beta(t)\quad(0\le t\le T)\\
        0\quad(t<0,T<t)
    \end{array}\right.,\quad
    \tilde{\gamma}(t)=\left\{\begin{array}{ll}
        \gamma(t)\quad(0\le t\le T)\\
        0\quad(t<0,T<t)
    \end{array}\right.
\end{array}\end{align*}
and $\alpha_\delta\coloneqq\eta_\delta\ast\tilde{\alpha},\;\beta_\delta\coloneqq\eta_\delta\ast\tilde{\beta},\;\gamma_\delta\coloneqq\eta_\delta\ast\tilde{\gamma}$, 
where $\eta_\delta(x)=\frac{1}{\delta}\eta\left(\frac{x}{\delta}\right)$ and $\ast$ is the convolution product. 

\section{Main theorems}\label{sec3}
\begin{theorem}[UAP for ODENet]\label{main_th}
    It holds that
    \begin{align*}
        \overline{\s(\F)}^{C^0(D;\R^N)}=\overline{\s(\h)}^{C^0(D;\R^N)},
    \end{align*}
    i.e., for given  $F\in\overline{\s(\F)}^{C^0(D;\R^N)}$ and $\epsilon>0$, there exists $h\in\h$ such that
    \begin{align*}
    \|F-S_h(T)\|_{C^0(D;\R^N)}<\epsilon.
    \end{align*}
\end{theorem}

\begin{theorem}\label{main_th_Lp}
    It holds that
    \begin{align*}
    \overline{\s(\F)}^{L^p(D;\R^N)}=\overline{\s(\h)}^{L^p(D;\R^N)},
    \end{align*}
    i.e., for given  $F\in\overline{\s(\F)}^{L^p(D;\R^N)}$ and $\epsilon>0$, there exists $h\in\h$ such that
    \begin{align*}
    \|F-S_h(T)\|_{L^p(D;\R^N)}<\epsilon.
    \end{align*}
\end{theorem}

\begin{remark}
    In the case where $N\ge2$ and $D$ is an open cube, since all measure-preserving maps are included in $\overline{\s(\F)}^{L^p(D;\R^N)}$ \cite[Corollary 1.1]{BG03}, it follows that all measure-preserving maps are contained in $\overline{\s(\h)}^{L^p(D;\R^N)}$.
\end{remark}

\begin{remark}\label{Rem:adv}
    Theorem \ref{main_th} claims that $S_f$ can be approximated by $S_{\alpha\odot\sigma(\beta\cdot x+\gamma)}$, not by\\ $S_{\sum_{l=1}^L\alpha_l\odot\sigma(\beta_l\cdot x+\gamma_l)}$. Since $\sigma$ has UAP, $f$ can be approximated by $\sum_{l=1}^L\alpha_l\odot\sigma(\beta_l\cdot x+\gamma_l)$ (cf. Lemma \ref{lem:UAP_Ndim}), one expects that $S_f$ can be approximated by $S_{\sum_{l=1}^L\alpha_l\odot\sigma(\beta_l\cdot x+\gamma_l)}$ (cf. \cite{TTIIO20}). On the other hand, $f$ cannot be approximated by $\alpha\odot\sigma(\beta\cdot x+\gamma)$ in general. Thus, Theorem \ref{main_th} is not obvious.
\end{remark}

Here we show that, in general, any $F\in C^0(D;\R^N)$ cannot be approximated by $S_f(T)$.
\begin{proposition}[\cite{hzh}]\label{prop2}
Let $N=1, D=[-1,1]$, and let $F:D\rightarrow\R$ satisfy
\begin{align*}
F(\xi)=-\xi\quad(\xi\in D).
\end{align*}
Then, we have $\|F - S_f(T)\|_{C^0(D;\R)}\ge 1$ for all $f\in\F$.
\end{proposition}
\begin{proof}
Suppose, for the sake of contradiction, that there exists $f\in\F$ such that $|S_f(T)-F|_{C^0(D)}<1$. Let $x_\pm(t)=(S_f(t))(\pm1)$ (in the same order of composition). Then we have $x_\pm(0)=\pm1$ and
\begin{align*}
    x_+(T) &< |x_+(T) - F(1)| + F(1) < 1-1 = 0, \\
    x_-(T) &> -|x_-(T) - F(-1)| + F(-1) > -1+1 = 0.
\end{align*}
By the intermediate value theorem, there exists $t^*\in(0,T)$ such that $x_-(t^*) = x_+(t^*)$. This contradicts the uniqueness of the solution to \eqref{ode} (see Theorem \ref{thm:Existence and Uniqueness of Solutions}).
\end{proof}

By Theorem \ref{main_th}, we also obtain the following theorem.

\begin{theorem}\label{main_th_2}
For all $f\in\F$ and $\epsilon>0$, there exists $s^{res}\in\s^{res}$ such that 
\begin{align*}
\|S_f(T)-s^{res}\|_{C^0(D;\R^N)}<\epsilon.
\end{align*}
\end{theorem}

\section{Preliminary results}\label{sec4}
We prepare some theorems and lemmas to prove the main theorems.

\begin{theorem}[Existence and uniqueness of solutions]\label{thm:Existence and Uniqueness of Solutions}
    Let $f:\R^N\times[0,T]\rightarrow\R^N$ be measurable and satisfy \eqref{asmp_fLip}. If there exists $\xi_0\in\R^N$ such that $f(\xi_0,\cdot)\in L^\infty(0,T)$, then there exists a unique solution $x(t)$ to \eqref{sol_ode}. In particular, if $f$ satisfies \eqref{asmp_fLip} and \eqref{asmp_fdiscr}, then there exist a unique solution $x(t)$ to \eqref{sol_ode}.
\end{theorem}
\begin{proof}
We prove Theorem \ref{thm:Existence and Uniqueness of Solutions} using Picard's successive approximation method.
Let $F_0=\|f(\xi_0,\cdot)\|_{L^\infty(0,T)}$. First, we show the existence of a solution $x(t)$. Consider the following sequence of functions:
\begin{align}\label{3_1}
    \left\{\begin{array}{ll}
        x_0(t)\coloneqq \xi_0\quad 
        (0\le t\le T)\\
        {\displaystyle x_n(t)\coloneqq \xi_0+\int^t_{0}f(x_{n-1}(s),s)ds\quad(0\le t\le T)\quad (n=1,2,\ldots).}
    \end{array}\right.
\end{align}
The sequence $\{x_n\}^\infty_{n=0}\subset C^0([0,T];\R^N)$ satisfies, for all $n=1,2,\ldots$ and $0\le t\le T$,
\begin{align}\label{2_7}
    |x_n(t)-x_{n-1}(t)|\le \frac{F_0\Lip (f)^{n-1}}{n!}t^n,
\end{align}
which can be proved by induction. Hence, it holds that for all $m\geq l\geq 1$, 
\begin{align*}
    |x_m(t)-x_l(t)|
    &\le \sum^m_{j=l+1}|x_j(t)-x_{j-1}(t)|
    \le \frac{F_0}{\Lip (f)}\sum^m_{j=l+1}\frac{\Lip (f)^jt^j}{j!}
    \le \frac{F_0}{\Lip (f)}\sum^\infty_{j=l+1}\frac{(\Lip (f)T)^j}{j!}\rightarrow 0
\end{align*}
as $l\rightarrow\infty$.
Therefore, $\{x_n\}^\infty_{n=1}$ is a Cauchy sequence in $C^0([0,T])$ and there exists $x\in C^0([0,T])$ such that
\begin{align*}
x_n\rightrightarrows x\quad \text{on}\quad [0,T] \quad\text{as}\quad n\rightarrow\infty.
\end{align*}
Since we have that
\begin{align*}
    \left |\int^t_{0}f(x_n(s),s)ds-\int^t_{0}f(x(s),s)ds\right |
    &\le \int^t_{0}|f(x_n(s),s)-f(x(s),s)|ds
    \le \Lip (f)\int^t_{0}|x_n(s)-x(s)|ds\\
    &\le \Lip (f)T\|x_n-x\|_{C^0([0,T])}\rightarrow 0
\end{align*}
as $n\rightarrow\infty$.
we obtain that $x(t)=\xi+\int^t_{0}f(x(s),s)ds\quad(0\le t\le T)$.\\
Second, we show the uniqueness of $x$ using contradiction. Suppose that there exist two solutions $x,\bar{x}$ satisfying $x\neq\bar{x}$. Then, it holds that
\begin{align*}
|x(t)-\bar{x}(t)|&\le\int^t_{0}|f(x(s),s)-f(\bar{x}(s),s)|ds\le \Lip (f)\int^t_{0}|x(s)-\bar{x}(s)|ds.
\end{align*}
By Gronwall's inequality, we have that
\begin{align*}
|x(t)-\bar{x}(t)|\le 0\cdot e^{\Lip (f)t}=0.
\end{align*}
It means $x=\bar{x}$. But this contradicts the assumption. Therefore, the solution to \eqref{sol_ode} is unique. 
\end{proof} 

\begin{lemma}\label{lem:range of solutions}
    Under the assumption of Theorem \ref{thm:Existence and Uniqueness of Solutions}, we have that 
    \begin{align*}
        |S_f(t)\xi_0- \xi_0|&\le M\quad (0\le t\le T),
    \end{align*}
    where $M=F_0Te^{\Lip (f)T}$ and $F_0=\|f(\xi_0,\cdot)\|_{L^\infty(0,T)}$.
\end{lemma}
\begin{proof}
Let $x(t) = S_f(t)\xi~(0 \le t \le T)$. It holds that 
\begin{align*}
    |x(t)-\xi_0|&\le \int^t_0|f(x(s),s)-f(\xi_0,s)|ds+\int^t_0|f(\xi_0,s)|ds
    \le\Lip (f)\int^t_0|x(s)-\xi_0|ds+\int^t_0F_0ds\\
    &\le \Lip (f)\int^t_0|x(s)-\xi_0|ds+F_0T.
\end{align*}
By Gronwall's inequality, we have 
\begin{align*}
    |x(t)-\xi_0|\le F_0Te^{\Lip (f)t}\le F_0Te^{\Lip (f)T}.
\end{align*}
\end{proof} 

By \eqref{asmp_fLip} and \eqref{asmp_fdiscr}, $f(\cdot,t)$ is continuous on $\R^N$ for all $0\le t \le T$, and $f \in L^\infty(0,T; C^0(D))$.
By Lemma \ref{lem:range of solutions}, it holds that for all $\xi\in D$,
\begin{align*}
    |S_f(t)\xi-\xi|\le \|f\|_{L^\infty(0,T; C^0(D))}Te^{\Lip (f)T}.
\end{align*}
Hence, if we define $E=\{\bar{\xi}\in\R^N~|\; 
    |\bar{\xi} - \xi|
    \le \|f\|_{L^\infty(0,T; C^0(D))}Te^{\Lip (f)T}, 
    \xi \in D
\}$, then it holds that 
\begin{align*}
    S_f(t)\xi\in E\quad(\xi\in D,\; 0\le t\le T).
\end{align*}

\begin{lemma}\label{lem:odenet_inequality}
    Let $f,g:\R^N\times[0,T]\rightarrow\R^N$ satisfy \eqref{asmp_fLip} and \eqref{asmp_fdiscr}. Let $E\subset\R^N$ be a bounded closed subset and satisfy that
    \begin{align*}
        E\supset\{S_g(t)\xi\; | \; \xi\in D,\; 0\le t\le T\}.
    \end{align*}
    Then, we have
    \begin{align*}
        \|S_f(t)-S_g(t)\|_{C^0(D)}\le\|f-g\|_{L^\infty(0,T; C^0(E))}T e^{\Lip (f)T}\quad (0\le t\le T).
    \end{align*}
\end{lemma}
\begin{proof}
Let $x_f(t)=S_f(t)\xi$ and $x_g(t)=S_g(t)\xi \;\left(0\le t \le T,\; \xi\in D\right)$. Then, it holds that for all $0\le t \le T$ and $\xi\in D$,
\begin{align*}
    |x_f(t)-x_g(t)|
    &= \left|\int^t_0f(x_f(s),s)-g(x_g(s),s)ds\right|\\
    &\le \int^t_0|f(x_f(s),s)-f(x_g(s),s)|ds
    +\int^t_0|f(x_g(s),s)-g(x_g(s),s)|ds\\
    &\le \Lip (f)\int^t_0|x_f(s)-x_g(s)|ds
    +\|f-g\|_{L^1(0,T; C^0(E))}.
\end{align*}
Hence, by Gronwall's inequality,
\begin{align*}
    |x_f(t)-x_g(t)|\le\|f-g\|_{L^1(0,T; C^0(E))}e^{\Lip (f)t}\le\|f-g\|_{L^\infty(0,T; C^0(E))}Te^{\Lip (f)T},
\end{align*}
which holds for all $\xi\in D$ and implies the conclusion.
\end{proof}

\begin{lemma}\label{lem:UAP_Ndim}
    Given $\delta>0,$ there exist $ K\in\N$ and $(\alpha_i,\beta_i,\gamma_i)^K_{i=1}\subset\R^N\times\R^{N\times N}\times\R^N$ such that for all $x\in D$,
    \begin{align*}
    {\underset{x\in D}{\max}}\left |f(x)-\sum^K_{i=1}\alpha_i\odot\boldsymbol{\sigma}(\beta_i x+\gamma_i)\right |\le \delta.
    \end{align*}
\end{lemma}
\begin{proof}
Let $f(x) = (f_1(x),f_2(x),\ldots, f_N(x))^\top$. Since $\sigma$ has UAP, for all $\delta>0$ and $l=1,2,\ldots, N$, there exist $K_l\in\N$ and $(\alpha^{(l)}_i,\beta^{(l)}_i,\gamma^{(l)}_i)^{K_l}_{i=1}\subset\R\times\R^{N}\times\R$ such that for all $x\in D$,
\begin{align*}
    \left|f_l(x)-\sum^{K_l}_{i=1}\alpha^{(l)}_i\sigma(\beta^{(l)\top}_i x+\gamma^{(l)}_i)\right|< \delta.
\end{align*}
Let $K\coloneqq {\underset{l=1,2,\ldots,L}{\max}}K_l$. We define $ \left( \hat{\alpha}^{(l)}_i,\hat{\beta}^{(l)}_i,\hat{\gamma}^{(l)}_i\right)\in\R\times\R^N\times\R$ as for $l=1,2,\ldots,L,$
\begin{align*}
    \left( \hat{\alpha}^{(l)}_i,\hat{\beta}^{(l)}_i,\hat{\gamma}^{(l)}_i\right)
    \coloneqq\left\{\begin{array}{ll}
        \left( \alpha^{(l)}_i,\beta^{(l)}_i,\gamma^{(l)}_i\right)\quad (i=1,2,\ldots,K_l)\\
        (0,0,0)\quad(i=K_{l}+1,K_{l}+2,\ldots,K)
    \end{array}\right.
\end{align*}
and we set for $i=1,2,\ldots,K$,
\begin{align*}
    \alpha_i&\coloneqq(\hat{\alpha}^{(1)}_i,\hat{\alpha}^{(2)}_i,...,\hat{\alpha}^{(N)}_i)^\top\in\R^N,\\
    \beta_i&\coloneqq(\hat{\beta}^{(1)}_i,\hat{\beta}^{(2)}_i,...,\hat{\beta}^{(N)}_i)^\top\in\R^{N\times N},\\
    \gamma_i&\coloneqq(\hat{\gamma}^{(1)}_i,\hat{\gamma}^{(2)}_i,...,\hat{\gamma}^{(N)}_i)^\top\in\R^N.
\end{align*}
Then, it holds that for all $x\in D$
\begin{align*}
    {\underset{x\in D}{\max}}\left|f(x)-\sum^K_{i=1}\alpha_i\odot\boldsymbol{\sigma}(\beta_i x+\gamma_i)\right|\le \delta.
\end{align*}
\end{proof}

\begin{theorem}[\cite{Picc}]\label{thm:piccinini}
    Let $f:\R^N\times\R\rightarrow\R^N$ satisfy $f(x,t+T)=f(x,t)$ for a.e. $(x,t)\in\R^N\times\R$ for a constant $T>0$. For $m\in\N$, we define
    \begin{align}
        f_m(x,t)&\coloneqq f(x,mt),\label{eq4}\\
        \bar{f}(x)&\coloneqq\frac{1}{T}\int^T_0f(x,t)dt.\label{eq6}
    \end{align}
    Then, for all $\xi\in D$,
    \[
        S_{f_m}(\cdot)\xi\rightrightarrows S_{\bar{f}}(\cdot)\xi\quad\text{on}\quad[0,T]\quad\text{as}\quad m\rightarrow\infty.
    \]
\end{theorem}

Since $f$ satisfies \eqref{asmp_fLip} and \eqref{asmp_fdiscr}, Theorem \ref{thm:piccinini} holds by \cite[Theorem 1]{Picc}. For the reader's convenience, we provide a proof of Theorem \ref{thm:piccinini} in Appendix \ref{app1}.

\begin{lemma}\label{lem:gronwall_ineq}
    Assume that there exists $B_0\subset\R^N$ such that $S_f(t)\xi\in B_0\;(0\le t \le T)$ for all $\xi\in D$. For $a>0$, we define $a$-neighborhood of $B_0$ as
    \begin{align*}
        B_a\coloneqq\left\{y\in\R^N \;|\; |x-y|<a,~ x\in B_0\right\}.
    \end{align*} 
    If $f,g\in C^0(\R^N;\R^N)$ satisfy $\|f-g\|_{C^0(\overline{B_a})}<\frac{a}{2Te^{\Lip (f)T}}$ and $\bar{\xi}\in D$ satisfies $|\xi-\bar{\xi}|<\frac{a}{2e^{\Lip (f)T}}$, then we get  $S_g(t)\bar{\xi}\in\overline{B_a}\;(0\le t \le T)$ and $|S_f(t)\xi-S_g(t)\bar{\xi}|\le a\;(0\le t \le T)$.
\end{lemma}
\begin{proof}
We prove $S_g(t)\bar{\xi}\in B_a\;(\bar{\xi}\in D,\; t\in[0,T])$ using contradiction. If we assume $S_g(t)\bar{\xi}\notin\overline{B_a}$, there exist $0<t_\ast\le T$ and $\bar{\xi}\in D$ such that
\begin{align*}
    \left\{\begin{array}{ll}
    S_g(t_\ast)\bar{\xi}\in\partial B_a\\
    S_g(t)\bar{\xi}\in B_a\quad\left(0\le t<t_\ast \right).
    \end{array}\right.
\end{align*}
We put $x(t)=S_f(t)\xi,\;\bar{x}(t)=S_g(t)\bar{\xi}$. Since $x(t),\bar{x}(t)\in\overline{B_a}$ for all $0\le t\le t_\ast $, we have
\begin{align*}
    |x(t)-\bar{x}(t)|
    &=\left|\xi+\int^t_0 f(x(s))ds-\bar{\xi}-\int^t_0 g(\bar{x}(s))ds\right|\\
    &\le |\xi-\bar{\xi}|+\int^t_0|f(x(s))-f(\bar{x}(s))|ds+\int^t_0|f(\bar{x}(s))-g(\bar{x}(s))|ds\\
    &\le|\xi-\bar{\xi}|+\Lip (f)\int^t_0|x(s)-\bar{x}(s)|ds+t\|f-g\|_{C^0(\overline{B_a})}
\end{align*}
for $0\le t\le t_\ast$. By Gronwall's inequality, we get
\begin{align*}
    |x(t)-\bar{x}(t)|
    &\le\left(|\xi-\bar{\xi}|+t\|f-g\|_{C^0(\overline{B_a})}\right)e^{\Lip (f)t}\quad(0\le t\le t_\ast).
\end{align*}
In particular, if $t=t_\ast$, we have
\begin{align*}
    |x(t_\ast)-\bar{x}(t_\ast)|
    &\le\left(|\xi-\bar{\xi}|+t_\ast\|f-g\|_{C^0(\overline{B_a})}\right)e^{\Lip (f)t_\ast}
    \le\left(|\xi-\bar{\xi}|+T\|f-g\|_{C^0(\overline{B_a})}\right)e^{\Lip (f)T}\\
    &<\left(\frac{a}{2e^{\Lip (f)T}}+\frac{a}{2e^{\Lip (f)T}}\right)e^{\Lip (f)T}=a,
\end{align*}
which contradicts the assumption $|x(t_\ast)-\bar{x}(t_\ast)|\geq a$. Therefore, Lemma \ref{lem:gronwall_ineq} was proved.
\end{proof}

\begin{lemma}\label{lem:bounded_mollifier}
    For $\alpha,\gamma\in L^\infty(0,T;\R^N)$, $\beta\in L^\infty(0,T;\R^{N\times N})$, and $\delta>0$, we define
    \begin{align*}
        h_\delta(x,t)&\coloneqq\alpha_\delta(t)\odot\boldsymbol{\sigma}(\beta_\delta(t) x(t)+\gamma_\delta(t)).
    \end{align*}
    Then, there exist $M_1$ and $M_2>0$ independent of $\delta$ such that
    \begin{align*}
        |S_{h_\delta}(t)\xi|&\le M_1\quad(\xi\in D,0\le t\le T),\\
        |\boldsymbol{\sigma}(\beta_\delta(t)S_{h_\delta}(t)\xi +\gamma_\delta(t))|
        &\le M_2\quad(\xi\in D,0\le t\le T).
    \end{align*}

\end{lemma}
\begin{proof}
First, $\alpha_\delta,\;\beta_\delta,\;\gamma_\delta$ are bounded. In fact,
\begin{align*}
    |\alpha_\delta(t)|&\le|(\eta_\delta\ast\tilde{\alpha})(t)|\le\int_\R\eta_\delta(t-s)|\tilde{\alpha}(t)|ds
    \le\|\alpha\|_{L^\infty(0,T;\R^N)}\int_\R\eta_\delta(\tau)d\tau
    =\|\alpha\|_{L^\infty(0,T;\R^N)}.
\end{align*}
Similarly, we have $|\beta_\delta(t)|\le\|\beta\|_{L^\infty(0,T;\R^{N\times N})}$ and $|\gamma_\delta(t)|\le\|\gamma\|_{L^\infty(0,T;\R^{N})}$. Let $\xi\in D$ be fixed and let $x_\delta(t)=S_{h_\delta}(t)\xi$. We have
\begin{align}\label{ineq:bdd4modify1}\begin{aligned}
    |x_\delta(t)-\xi|
    &=\left|\int^t_0\alpha_\delta(s)\odot\boldsymbol{\sigma}(\beta_\delta(t) x_\delta(s)+\gamma_\delta(s))ds\right|
    \le \|\alpha\|_{L^\infty(0,T;\R^N)}\int^t_0|\boldsymbol{\sigma}(\beta_\delta(t) x_\delta(s)+\gamma_\delta(s))|ds\\
    &\le \|\alpha\|_{L^\infty(0,T;\R^N)}\int^t_0|\boldsymbol{\sigma}(\beta_\delta(t) x_\delta(s)+\gamma_\delta(s))-\boldsymbol{\sigma}(\beta_\delta(t)\xi+\gamma_\delta(s))|ds \\&\quad~ +\|\alpha\|_{L^\infty(0,T;\R^N)}\int^t_0|\boldsymbol{\sigma}(\beta_\delta(t)\xi+\gamma_\delta(s))|ds\\
    &\le \|\alpha\|_{L^\infty(0,T;\R^N)}\Lip (\sigma)\|\beta\|_{L^\infty(0,T;\R^{N\times N})}\int^t_0|x_\delta(s)-\xi|ds\\
    &\quad~ + \|\alpha\|_{L^\infty(0,T;\R^N)}\int^t_0|\boldsymbol{\sigma}(\beta_\delta(t)\xi+\gamma_\delta(s))|ds.
\end{aligned}\end{align}
Since we have
\begin{align*}
    |\beta_\delta(s)\xi+\gamma_\delta(s)|
    &\le\|\beta_\delta\|_{L^{\infty}(0,T;\R^{N\times N})}|\xi|+\|\gamma_\delta\|_{L^{\infty}(0,T;\R^N)}
    \le\|\beta\|_{L^{\infty}(0,T;\R^{N\times N})}|\xi|+\|\gamma\|_{L^{\infty}(0,T;\R^N)}
    \le \tilde{R},
\end{align*}
where $\tilde{R}\coloneqq\|\beta\|_{L^{\infty}(0,T;\R^{N\times N})}{\underset{\xi\in D}{\max}}|\xi|+\|\gamma\|_{L^{\infty}(0,T;\R^N)}$, it holds that
\begin{align}\label{ineq:bdd4modify2}
    \int^t_0|\boldsymbol{\sigma}(\beta_\delta(t)\xi+\gamma_\delta(s))|ds
    &\le \int^t_0\|\boldsymbol{\sigma}\|_{C^0(B_{\tilde{R}})}ds
    \le T\|\boldsymbol{\sigma}\|_{C^0(B_{\tilde{R}})}.
\end{align}
From \eqref{ineq:bdd4modify1} and \eqref{ineq:bdd4modify2}, by Gronwall's inequality, we get
\begin{align*}
    |x_\delta(t)|\le \xi + T\|\alpha\|_{L^\infty(0,T;\R^N)}\|\boldsymbol{\sigma}\|_{C^0(B_{\tilde{R}})}e^{\|\alpha\|_{L^\infty(0,T;\R^N)}\|\beta\|_{L^\infty(0,T;\R^{N\times N})}\Lip (\sigma)t}\le M_1
\end{align*}
where $M_1\coloneqq \underset{\xi\in D}{\max}|\xi| + T\|\alpha\|_{L^\infty(0,T;\R^N)}\|\boldsymbol{\sigma}\|_{C^0(B_{\tilde{R}})}e^{\|\alpha\|_{L^\infty(0,T;\R^{N})}\|\beta\|_{L^\infty(0,T;\R^{N\times N})}\Lip (\sigma)T}$. Since $x_\delta(t)\;(0\le t\le T),\; \beta_\delta,\;\gamma_\delta$ are bounded and $\boldsymbol{\sigma}\in C^0(\R^N;\R^N)$, there exists $M_2>0$ such that
\begin{align*}
    |\boldsymbol{\sigma}(\beta_\delta(t) x_\delta(t)+\gamma_\delta(t))|&\le M_2
\end{align*}
for all $0\le t\le T$.
\end{proof}

\section{Proofs of the main theorems}\label{sec5}
For Theorems \ref{main_th} and \ref{main_th_Lp}, it is enough to prove that $F\in\s(\F)$ belongs to $\overline{\s(\h)}^{C^0(D;\R^N)}$.
We divide the proof of Theorem \ref{main_th} into three steps. For a given $\epsilon>0$, we aim to derive the following inequality: for all $0\le t\le T$,
\begin{align*}
    &\,\|S_f(t)-S_h(t)\|_{C^0(D;\R^N)}\\
    \le&\,\|S_f(t)-S_{f_L}(t)\|_{C^0(D;\R^N)}+\|S_{f_L}(t)-S_{h_L}(t)\|_{C^0(D;\R^N)}
    +\|S_{h_L}(t)-S_h(t)\|_{C^0(D;\R^N)}\\
    <&\,\frac{\epsilon}{3}+\frac{\epsilon}{3}+\frac{\epsilon}{3}=\epsilon,
\end{align*}
where $f_L$ is a piecewise constant approximation of $f$ in the time direction, $h_L$ is an approximation of $f_L$ that can be expressed as $\alpha\odot\boldsymbol{\sigma}(\beta x+\gamma)$ at each time, and $h$ is a smooth approximation of $h_L$.

\subsection{Proof of $\|S_f(t)-S_{f_L}(t)\|_{C^0(D;\R^N)}<\frac{\epsilon}{3}$}\label{pr_1}
Let $L\in\N$, $\tau\coloneqq\frac{T}{L}$, $t_l\coloneqq l\tau\;(l=0,1,\ldots,L)$, $f^{(l)}(x)\coloneqq f(x,t_l)$, and $ f_L (x,t)\coloneqq f^{(l)}(x)\;(t_{l-1}<t\le t_l,\; l=1,2,\ldots,L)$. By Lemma \ref{lem:range of solutions}, we know that $S_f(t)\xi,\; S_{f_L}(t)\xi\in B_0=\overline{B_R}$ for all $\xi\in D$ and $0\le t\le T$, where $R={\underset{\xi\in D}{\max}}|\xi|+Te^{\Lip (f)T}\|f\|_{L^\infty(0,T;C^0(D))}$.
By Lemma \ref{lem:range of solutions} and the uniform continuity of $f$, for all given $\epsilon>0$, there exists $L\in\N$ such that
\begin{align*}
    \|f-f_L\|_{L^\infty(0,T; C^0(B_0))}<\frac{\epsilon}{3Te^{\Lip (f)T}}.
\end{align*}
Therefore, by Lemma \ref{lem:odenet_inequality}, it holds that for all $0\le t\le T$,
\begin{align*}
    \|S_f(t)-S_{f_L}(t)\|_{C^0(D)}
    &\le t\|f-f_L\|_{L^\infty(0,T; C^0(B_0))}e^{\Lip (f)T}
    < \frac{\epsilon}{3Te^{\Lip (f)T}}Te^{\Lip (f)T}= \frac{\epsilon}{3}.
\end{align*}

From now on, this $L$ will be fixed.

\subsection{Proof of $\|S_{f_L}(t)-S_{h_L}(t)\|_{C^0(D;\R^N)}<\frac{\epsilon}{3}$}\label{pr_2}

Let $\xi\in D$. We introduce two notations:
\begin{align*}
    b_l=\frac{\epsilon}{3(4e^{\Lip (f)\tau})^{L-l}},\quad
    \xi_l=\left\{\begin{array}{ll}
        \xi\quad&(l=0)\\
        S_{f^{(l)}}(\tau)\xi_{l-1}\quad&(l=1,2,\ldots,L).
    \end{array}\right.
\end{align*}
We inductively construct functions $h^{(l)}:\R^N\times[0,\tau]\rightarrow\R^N$ and subsequently the function $h_L:\R^N\times[0,T]\rightarrow\R^N$.
We aim to find $h^{(l)}$ $(l=1,2,\ldots,L)$ such that if $\zeta_{l-1}\in\R^N$ satisfies
\begin{align}\label{ineq:b_l-1}
    |\xi_{l-1}-\zeta_{l-1}|\le b_{l-1},  
\end{align}
then $\zeta_l=S_{h^{(l)}}(\tau)\zeta_{l-1}$ satisfies
\begin{align}\label{ineq:b_l}
    |\xi_l-\zeta_l| \le b_l.
\end{align}

For $l=1,2,\ldots,L$, assume that $\zeta_{l-1}\in\R^N$ satisfies \eqref{ineq:b_l-1}.
Since $\sigma$ has universal approximation property, by Lemma \ref{lem:UAP_Ndim}, there exist $K_l\in\N$ and $(\alpha_i^{(l)},\beta_i^{(l)},$ $\gamma_i^{(l)})^{K_l}_{i=1}\subset\R^N\times\R^{N\times N}\times\R^N$ such that
\begin{equation*}
    \|f^{(l)}-\bar{g}_l\|_{C^0(\overline{B_{b_l}})}\le\frac{b_{l-1}}{\tau}=\frac{b_l}{4\tau e^{\Lip (f)\tau}},
\end{equation*}
where $B_{b_l}$ is $b_l$-neighborhood of $B_0$ and 
\begin{align*}
    \bar{g}_l(x)&\coloneqq \sum^{K_l}_{i=1}\alpha^{(l)}_i\odot\boldsymbol{\sigma}(\beta^{(l)}_i x+\gamma^{(l)}_i).
\end{align*}
By Lemma \ref{lem:gronwall_ineq}, we have 
\begin{align*}
    |S_{f^{(l)}}(t)\xi_{l-1}-S_{\bar{g}_l}(t)\zeta_{l-1}|\le\frac{1}{2}b_l\quad(0\le t\le \tau).
\end{align*}
On the other hand, let a $\tau$-periodic function  $g_l:\R^N\times\R\rightarrow\R^N$ and a function $g_{l,m}:\R^N\times[0,\tau]\rightarrow\R^N$ be defined by
\begin{align*}
    g_l(x,t)&\coloneqq K_l\alpha^{(l)}_i\odot\boldsymbol{\sigma}(\beta^{(l)}_i x+\gamma^{(l)}_i)\quad\left(\frac{i-1}{K_l}\tau<t\le\frac{i}{K_l}\tau,i=1,2,\ldots,K_l\right),\\
    g_{l,m}(x,t)&\coloneqq g_l(x,mt)\quad(x\in\R^N,0\le t\le \tau,l=1,2,\ldots,L,m\in\N).
\end{align*}
By Theorem \ref{thm:piccinini}, there exists $m_l\in\N$ such that
\begin{align*}
    |S_{\bar{g}_l}(t)\zeta_{l-1}-S_{g_{l,m_l}}(t)\zeta_{l-1}|\le\frac{1}{2}b_l\quad(0\le t\le \tau).
\end{align*}
Hence, if we set $h^{(l)}=g_{l,m_l}$ and $\zeta_l=S_{h^{(l)}}(\tau)\zeta_{l-1}$, then we get \eqref{ineq:b_l}:
\begin{align*}
    |\xi_l-\zeta_l|
    &=|S_{f^{(l)}}(\tau)\xi_{l-1}-S_{h^{(l)}}(\tau)\zeta_{l-1}|
    \le |S_{f^{(l)}}(\tau)\xi_{l-1}-S_{\bar{g}_l}(\tau)\zeta_{l-1}|+|S_{\bar{g}_l}(\tau)\zeta_{l-1}-S_{h^{(l)}}(\tau)\zeta_{l-1}|
    \le b_l.
\end{align*}

Let
\begin{align*}
h_L (x,t)&\coloneqq h^{(l)}(x,t-t_{l-1})\quad(x\in\R^N,t_{l-1}< t\le t_l,l=1,2,\ldots,L).
\end{align*}
If we put $\zeta_l=S_{h_L}(t)\xi$ for all $l=0,1,\ldots,L$, then we get inductively for $t_{l-1} \le t \le t_l$ with $l=1,2,\ldots,L$:
\begin{align*}
|\xi_l-\zeta_l|\le b_l(\le b_L),\quad
|S_{f_L}(t)\xi-S_{h_L}(t)\xi|\le b_l(\le b_L).
\end{align*}
Especially, we have
\begin{align*}
|S_{f_L}(t)\xi-S_{h_L}(t)\xi|\le b_L=\frac{\epsilon}{3}\quad(0\le t\le T)
\end{align*}
for all $\xi$. Therefore, we obtain
\begin{align*}
\|S_{f_L}(t)-S_{h_L}(t)\|_{C^0(D)}<\frac{\epsilon}{3}\quad(0\le t\le T).
\end{align*}

\subsection{Proof of $\|S_{h_L}(t)-S_{h}(t)\|_{C^0(D;\R^N)}<\frac{\epsilon}{3}$}\label{pr_3}
By definition of $h_L, h^{(l)}, g_{l,m}$, and $g_l$, there exist $\alpha_L,\gamma_L\in L^\infty(0,T;\R^N)$ and $\beta_L\in L^\infty(0,T;\R^{N\times N})$ such that
\[
    h_L (x,t)\coloneqq\alpha_L (t)\odot\boldsymbol{\sigma}(\beta_L (t) x+\gamma_L (t))\;
    \quad(0\le t\le T,\; x\in\R^N).
\] 
We define $h\in C^0(\R^N\times[0,T];\R^N)$ as
\begin{align*}
    h(x,t)&\coloneqq\alpha_{L,\delta}(t)\odot\boldsymbol{\sigma}(\beta_{L,\delta}(t) x+\gamma_{L,\delta}(t))
    \quad(x\in\R^N,0\le t\le T).
\end{align*}
By Lemma \ref{lem:bounded_mollifier}, there exist $M_1$ and $M_2$ such that for all $\xi\in D$ and $0\le t\le T$,
\begin{align*}
    |S_h(t)\xi|\le M_1,\quad
    |\boldsymbol{\sigma}(\beta_{L,\delta}(t) S_h(t)\xi+\gamma_{L,\delta}(t))|\le M_2.
\end{align*}
Given $\epsilon>0,$ there exists $\delta>0$ that satisfies
\begin{align}\label{ineq:errabg}
    \|\alpha_{L,\delta} - \alpha_L\|_{L^1(0,T)}<\epsilon',\quad
    \|\beta_{L,\delta} - \beta_L\|_{L^1(0,T)}<\epsilon',\quad
    \|\gamma_{L,\delta} - \gamma_L\|_{L^1(0,T)}<\epsilon',
\end{align}
where $\epsilon'=\frac{\epsilon}{3M_3e^{M_4T}}$, $M_3\coloneqq M_2+(M_1+1)\|\alpha_L\|_{L^\infty(0,T;\R^N)} \Lip (\sigma)$, and $M_4\coloneqq\|\alpha_L\|_{L^\infty(0,T;\R^N)}$\\ $\|\beta_L\|_{L^\infty(0,T;\R^{N\times N})} \Lip (\sigma)$.
If we set $\alpha=\alpha_{L,\delta}$, $\beta=\beta_{L,\delta}$, $\gamma=\gamma_{L,\delta}$, $x_L (t)=S_{h_L}(t)\xi$ and $x(t)=S_h(t)\xi$, then it holds that for all $0\le t\le T$,
\begin{align*}
    |x(t)-x_L (t)|
    = \left|\int^t_0 \left(h(x(s),t)-h_L (x_L (s),s)\right)ds\right|
    \le A+B,
\end{align*}
where $A\coloneqq\int^t_0(\alpha(s)-\alpha_L (s))\odot\boldsymbol{\sigma}(\beta(s) x(s)+\gamma(s))ds$ and $B\coloneqq\int^t_0\alpha_L (s)\odot\{\boldsymbol{\sigma}(\beta(s) x(s)+\gamma(s))-\boldsymbol{\sigma}(\beta_L (s) x_L (s)+\gamma_L (s))\}ds$. 
Since we have 
\begin{align*}
    |A|&\le\int^t_0M_2|\alpha(s)-\alpha_L (s)|ds
    \le M_2\|\alpha-\alpha_L\|_{L^1(0,T)},\\
    |B|&\le\|\alpha_L\|_{L^\infty(0,T;\R^N)}\int^t_0
    |\boldsymbol{\sigma}(\beta(s) x(s)+\gamma(s))
    -\boldsymbol{\sigma}(\beta_L (s) x_L (s)+\gamma_L (s))|ds\\
    &\le\|\alpha_L\|_{L^\infty(0,T;\R^N)} \Lip (\sigma)\int^t_0\left(
    |(\beta(s)-\beta_L (s))x(s)|
    +|\beta_L (s)(x(s)-x_L (s))|
    +|\gamma(s)-\gamma_L (s)|\right)ds\\
    &\le\|\alpha_L\|_{L^\infty (0,T;\R^N)}\Lip (\sigma)M_1\|\beta-\beta_L\|_{L^1(0,T)}
    +\|\alpha_L\|_{L^\infty(0,T;\R^N)} \Lip (\sigma)\|\gamma-\gamma_L\|_{L^1(0,T)}
    \\&\quad
    +\|\alpha_L\|_{L^\infty(0,T;\R^N)}\|\beta_L\|_{L^\infty(0,T;\R^{N\times N})} \Lip (\sigma)\int^t_0|x(s)-x_L (s)|ds,
\end{align*}
by \eqref{ineq:errabg} and the definition of $M_3$ and $M_4$, we obtain
\begin{align*}
    |x(t)-x_L (t)|\le M_3\epsilon'+M_4\int^t_0|x(s)-x_L (s)|ds.
\end{align*}
By Gronwall's inequality, we have that for all $0\le t\le T$ and $\xi\in D$,
\begin{align*}
    |S_h(t)\xi-S_{h_L}(t)\xi|=|x(t)-x_L (t)|
    \le M_3\epsilon' e^{M_4t} \le M_3\epsilon' e^{M_4T}.
\end{align*}
Therefore, it holds that
\begin{align*}
    \|S_{h_L}(t)-S_h(t)\|_{C^0(D)}<\frac{\epsilon}{3}\quad(0\le t\le T).
\end{align*}

\subsection{Proof Summary of Theorem \ref{main_th}}
From subsection \ref{pr_1}, \ref{pr_2}, and \ref{pr_3}, given $\epsilon>0$, there exist $\alpha,\gamma\in C^0([0,T];\R^N)$ and $\beta\in C^0([0,T];\R^{N\times N})$ such that  for all $0\le t\le T$,
\begin{align*}
    &\quad\,\|S_f(t)-S_h(t)\|_{C^0(D;\R^N)}\\
    &\le\|S_f(t)-S_{f_L}(t)\|_{C^0(D;\R^N)}+\|S_{f_L}(t)-S_{h_L}(t)\|_{C^0(D;\R^N)}
    +\|S_{h_L}(t)-S_h(t)\|_{C^0(D;\R^N)}\\
    &<\frac{\epsilon}{3}+\frac{\epsilon}{3}+\frac{\epsilon}{3}=\epsilon,
\end{align*}
where $h(x,t)=\alpha(t)\odot\boldsymbol{\sigma}(\beta(t) x+\gamma(t))~(x\in\R^N,0\le t\le T)$.
Therefore, we have proved Theorem \ref{main_th}.

\subsection{Proof of Theorem \ref{main_th_2}}
We prepare the following lemma to prove Theorem \ref{main_th_2}.
\begin{lemma}\label{lm_res}
For $h\in\h$, given  $\epsilon>0$, there exists $L_0\in\N$ such that if $L>L_0$ and $s,t\in[0,T]$ satisfy $|s-t|\le\frac{T}{L}$ then
\begin{align*}
|h(S_h(s)\xi,s)-h(S_h(t)\xi,t)|<\epsilon
\end{align*}
for all $\xi\in D$.
\end{lemma}
\begin{proof}
Let $x(t;\xi)=S_h(t)\xi$. By the definition of $\h$ and \eqref{asmp_sigma}, we obtain for $s,t\in[0,T]$,
\begin{align*}
    &\quad|h(x(s;\xi),s)-h(x(t;\xi),t)|\\
    &=|\alpha(s)\odot\boldsymbol{\sigma}(\beta(s)x(s;\xi)+\gamma(s))-\alpha(t)\odot\boldsymbol{\sigma}(\beta(t)x(t;\xi)+\gamma(t))|\\
    &\le|(\alpha(s)-\alpha(t))\odot\boldsymbol{\sigma}(\beta(s)x(s;\xi)+\gamma(s))|
    +|\alpha(t)\odot\left\{\boldsymbol{\sigma}(\beta(s)x(s;\xi)+\gamma(s))-\boldsymbol{\sigma}(\beta(t)x(t;\xi)+\gamma(t))\right\}|\\
    &\le |\boldsymbol{\sigma}(\beta(s)x(s;\xi)+\gamma(s))| |\alpha(s)-\alpha(t)|
    + M_\alpha|\boldsymbol{\sigma}(\beta(s)x(s;\xi)+\gamma(s))-\boldsymbol{\sigma}(\beta(t)x(t;\xi)+\gamma(t))|,
\end{align*}
where $M_\alpha=\|\alpha\|_{C^0([0,T];\R^N)}$. By Lemma \ref{lem:bounded_mollifier}, there exists $M>0$ such that
\begin{align*}
    |\boldsymbol{\sigma}(\beta(s)x(s;\xi)+\gamma(s))|\le M\quad(\xi\in D,\;0\le s\le T).
\end{align*}
If we set $y(s)\coloneqq\boldsymbol{\sigma}(\beta(s)x(s;\xi)+\gamma(s))$, then $y\in C^0([0,T]\times D;\R^N)$. Hence, $\alpha$ is uniformly continuous on $[0,T]$ and $y$ is uniformly continuous on $[0,T
]\times D$ because $[0,T]$ and $[0,T]\times D$ are compact and  $\alpha\in C^0([0,T];\R^N),\; y\in C^0([0,T]\times D;\R^N)$. Therefore, there exists $L_0\in\N$ such that if $L>L_0$ and $|s-t|<\frac{T}{L}$ then
\begin{align*}
    |\alpha(s)-\alpha(t)|\le \epsilon',\quad
    |y(s)-y(t)|\le\epsilon',
\end{align*}
where $\epsilon'=\frac{\epsilon}{M+M_\alpha}$ and we obtain
\begin{align*}
    |h(x(s;\xi),s)-h(x(t;\xi),t)|
    &\le M|\alpha(s)-\alpha(t)|+M_\alpha|y(s)-y(t)|
    \le M\epsilon'+M_\alpha\epsilon'=\epsilon.
\end{align*}
\end{proof}

\begin{proof}[\textbf{Proof of Theorem \ref{main_th_2}}]
By Theorem \ref{main_th}, there exist $\alpha,\gamma\in C^0([0,T];\R^N)$ and $\beta\in C^0([0,T];\R^{N\times N})$ such that
\begin{align*}
    \|S_f(T)-S_h(T)\|_{C^0(D;\R^N)}<\frac{\epsilon}{2}
\end{align*}
where $h(x,t)=\alpha(t)\odot\boldsymbol{\sigma}(\beta(t)x+\gamma(t))$. Let $L\in\N$ and let $t_l=\frac{l}{L}T\;(l=0,1,\ldots,L)$. For $l=0,1,\ldots,L$, we set $\alpha_l=\alpha(t_l)$, $\beta_l=\beta(t_l)$, $\gamma_l=\gamma(t_l)$, $h^{(l)}(x)=\alpha_l\odot\boldsymbol{\sigma}(\beta_l x+\gamma_l)$, and $z^{(l)}(\xi)\coloneqq x(t_l;\xi)-x^{(l)}(\xi)$, where $x(t;\xi)=S_h(t)\xi$ and $x^{(l)} = S^{res}_{\{\Delta t h^{(l)}\}_{l=0}^{L-1}}(l)$. It holds that for all $l=0,1,\ldots, L-1,$
\begin{align*}
    |z^{(l+1)}(\xi)|
    &=\left|\left(x(t_l;\xi)+\int^{t_{l+1}}_{t_l}h(x(s;\xi),s)ds\right)-\left(x^{(l)}(\xi)+\Delta t h^{(l)}(x^{(l)}(\xi))\right)\right|\\
    &=\left|z^{(l)}(\xi)+\int^{t_{l+1}}_{t_l}\left(h(x(s;\xi),s)-h(x^{(l)}(\xi),t_l)\right)ds\right|\\
    &\le |z^{(l)}(\xi)|+\int^{t_{l+1}}_{t_l}|h(x(s;\xi),s)-h(x(t_l;\xi),t_l)|ds+\int^{t_{l+1}}_{t_l}|h(x(t_l;\xi),t_l)-h(x^{(l)}(\xi),t_l)|ds.
\end{align*}
By \eqref{asmp_sigma}, we obtain
\begin{align*}
    |h(x(t_l;\xi),t_l)-h(x^{(l)}(\xi),t_l)|
    &=|\alpha(t_l)\odot\{ \boldsymbol{\sigma}(\beta(t_l)x(t_l;\xi)+\gamma(t_l))-\boldsymbol{\sigma}(\beta(t_l)x^{(l)}(\xi)+\gamma(t_l)) \}|\\
    &\le M_\alpha| \boldsymbol{\sigma}(\beta(t_l)x(t_l;\xi)+\gamma(t_l))-\boldsymbol{\sigma}(\beta(t_l)x^{(l)}(\xi)+\gamma(t_l)) |\\
    &\le M_\alpha \Lip (\sigma)|(\beta(t_l)x(t_l;\xi)+\gamma(t_l))
    -(\beta(t_l)x^{(l)}(\xi)+\gamma(t_l))|\\
    &\le  C|x(t_l;\xi)-x^{(l)}(\xi)|= C|z^{(l)}(\xi)|,
\end{align*}
where $C=M_\alpha M_\beta \Lip (\sigma),\; M_\alpha=\|\alpha\|_{C^0([0,T];\R^N)},$ and $M_\beta=\|\beta\|_{C^0([0,T];\R^{N\times N})}$. By Lemma \ref{lm_res}, there exists $L\in\N$ such that
\begin{align*}
|h(x(s;\xi),s)-h(x(t_l;\xi),t_l)|<\epsilon'\quad(t_l\le s\le t_{l+1},\; l=0,1,\ldots,L-1),
\end{align*}
where $\epsilon'=\frac{C\epsilon}{2e^{CT}}$, and hence
\begin{align*}
|z^{(l+1)}(\xi)|&\le|z^{(l)}(\xi)|+\Delta t\epsilon'+C\Delta t|z^{(l)}(\xi)|=(1+C\Delta t)|z^{(l)}(\xi)|+\Delta t\epsilon'.
\end{align*}
Furthermore, it holds that for all $l=1,2,\ldots,L$,
\begin{align*}
    |z^{(l)}(\xi)|&\le (1+C\Delta t)^l|z^{(0)}|+
    \sum^{l-1}_{m=0}(1+C\Delta t)^m\Delta t\epsilon'
    = (1+C\Delta t)^l|z^{(0)}|+\frac{((1+C\Delta t)^l-1)\Delta t\epsilon'}{1+C\Delta t-1}\\
    &\le (1+C\Delta t)^l|z^{(0)}|+\frac{(1+C\Delta t)^l\Delta t\epsilon'}{C\Delta t}
    =  \left(1+\frac{CT}{L}\right)^L\left\{|z^{(0)}|+\frac{\epsilon'}{C}\right\}.
\end{align*}
Since $z^{(0)}=\xi-\xi=0$ and $\left(1+\frac{x}{n}\right)^n\le e^x\quad(x>0, n\in \N),$
we obtain
\begin{align*}
    |z^{(l)}(\xi)|&\le e^{CT}\left\{|z^{(0)}|+\frac{\epsilon'}{C}\right\}\le  e^{CT}\frac{\epsilon'}{C}=\frac{\epsilon}{2}.
\end{align*}
Therefore, given $\epsilon>0$, there exists $L\in\N$ such that
\begin{align*}
    |S_h(T)\xi-x^{(L)}(\xi)|=|x(T;\xi)-x^{(L)}(\xi)|<\frac{\epsilon}{2}
\end{align*}
for all $\xi\in D$ and hence, the mapping $S^{res}_{\{\Delta t h^{(l)}\}_{l=0}^{L-1}}(L):D\ni\xi\mapsto x^{(L)}(\xi)\in\R^N$ satisfies that 
\begin{align*}
    &\quad\,\|S_f(T)-S^{res}_{\{\Delta t h^{(l)}\}_{l=0}^{L-1}}(L)\|_{C^0(D;\R^N)}
    =\|S_f(T)-x^{(L)}\|_{C^0(D;\R^N)}\\
    &\le\|S_f(T)-S_h(T)\|_{C^0(D;\R^N)}+\|S_h(T)-x^{(L)}\|_{C^0(D;\R^N)}
    <\epsilon.
\end{align*}
\end{proof}

\appendix
\section{Proof of Theorem \ref{thm:piccinini}}
\label{app1}

\begin{proof}[Proof of Theorem \ref{thm:piccinini}]
Let $\xi\in D$ be fixed and let $x_m(t)=S_{f_m}(t)\xi$, $\bar{x}(t)=S_{\bar{f}}(t)\xi$, $M=\|f\|_{L^\infty(0,T;C^0(D))}$\\ $Te^{\Lip (f)T}$, $E=\{ y\in\R^N\mid |x-y|\le M$ for some $x\in D\}$, and $C_0=\|f\|_{L^\infty(0,T;C^0(E))}$. For all $\epsilon>0$, if $0\le t_1,t_2\le T$ satisfy that $|t_1-t_2|<\delta_\epsilon\coloneqq\frac{\epsilon}{C_0}$, then we get,
\begin{align}\label{ineq0}
    |\bar{x}(t_2)-\bar{x}(t_1)|
    &\le \int^{t_2}_{t_1}|\bar{f}(\bar{x}(s),s)|ds
    =C_0|t_2-t_1|
    <C_0\frac{\epsilon}{C_0}
    =\epsilon.
\end{align}
Here, we take $m_\epsilon \in \N$ satisfying $\frac{T}{m_\epsilon}<\delta_\epsilon$ and $\frac{2TC_0}{m_\epsilon}<\epsilon$. Let $t>0$ be fixed and let $m>m_\epsilon,$ $N=\left\lfloor \frac{t}{T}m\right\rfloor+1,$ $ t_i=i\frac{T}{m}$ , $\bar{x}_i=\bar{x}(t_i)\:(i=0,1,\ldots,N-1),\; t_N=t$ and $\bar{x}_N=\bar{x}(t_N)$. Then, it holds that
\begin{align*}
    \bar{x}(t)-x_m(t)
    &=\int^t_{0}\left(\bar{f}(\bar{x}(\tau))-f(x_m(\tau),m\tau)\right)d\tau
    =\sum^N_{i=1}(I_{1,i}+I_{2,i}+I_{3,i})+I_4,
\end{align*}
where
\begin{align*}
    I_{1,i}&=\int^{t_i}_{t_{i-1}}\left(\bar{f}(\bar{x}(\tau))-\bar{f}(\bar{x}_i)\right)d\tau,\quad 
    I_{2,i}=\int^{t_i}_{t_{i-1}}\left(\bar{f}(\bar{x}_i)-f(\bar{x}_i,m\tau)\right)d\tau,\\
    I_{3,i}&=\int^{t_i}_{t_{i-1}}\left(f(\bar{x}_i,m\tau)-f(\bar{x}(\tau),m\tau)\right)d\tau,\quad
    I_4=\int^t_0\left(f(\bar{x}(\tau),m\tau)-f(x_m(\tau),m\tau)\right)d\tau.
\end{align*}
By \eqref{asmp_fLip}, \eqref{eq6}, and \eqref{ineq0}, we have that for $i=1,2,\ldots,N$ and $t_{i-1}\le\tau\le t_i$,
\begin{align*}
    |\bar{f}(\bar{x}(\tau))-\bar{f}(\bar{x}_i)|
    &=\left|\frac{1}{T}\int^T_0(f(\bar{x}(\tau),s)-f(\bar{x}_i, s))ds\right|
    \le\frac{1}{T}\int^T_0\Lip (f)|\bar{x}(\tau)-\bar{x}_i|ds
    =\Lip (f)\epsilon,
\end{align*}
which implies that for all $i=1,2,\ldots,N$, 
\begin{align}\begin{aligned}
    |I_{1,i}|
    &\le\int^{t_i}_{t_{i-1}}|\bar{f}(\bar{x}(\tau))-\bar{f}(\bar{x}_i)|d\tau 
    \le \int^{t_i}_{t_{i-1}}\Lip (f)|\bar{x}(\tau)-\bar{x}_i|d\tau
    \le \Lip (f)|t_i-t_{i-1}|\epsilon.\label{eq7}
\end{aligned}\end{align}
By \eqref{asmp_fLip}, it holds that for all $i=1,2,\ldots,N$,
\begin{align}\begin{aligned}
    |I_{3,i}|
    &\le \int^{t_i}_{t_{i-1}}|f(\bar{x}_i, m\tau)-f(\bar{x}(\tau),m\tau)|d\tau
    \le \int^{t_i}_{t_{i-1}}\Lip (f)|\bar{x}_i-\bar{x}(\tau)|d\tau\le \Lip (f)|t_i-t_{i-1}|\epsilon.\label{eq8}
\end{aligned}\end{align}
By periodicity of $f$, we get for all $i=1,2,\ldots,N-1$ and $x\in\R^N$, 
\begin{align*}
    \int^{t_i}_{t_{i-1}}f(x,m\tau)d\tau&=\int^{\frac{T}{m}}_0f(x,m\tau)d\tau=\frac{1}{m}\int^T_0f(x,\tau')d\tau' =\frac{T}{m}\bar{f}(x),
\end{align*}
where $\tau'=m\tau$, and hence, we have that for all $i=1,2,\ldots,N-1$,
\begin{align*}
    I_{2,i}=\int^{t_i}_{t_{i-1}}\left(\bar{f}(\bar{x}_i)-f(\bar{x}_i,m\tau)\right)d\tau=0.
\end{align*}
When $i=N$, we have
\begin{align}
    |I_{2,N}|
    \le\int^{t_N}_{t_{N-1}}|\bar{f}(\bar{x}_N)-f(\bar{x}_N,m\tau)|d\tau
    \le 2C_0(t_N-t_{N-1})\le \frac{2TC_0}{m},\label{eq9}
\end{align}
where we have used $t_N-t_{N-1}<\frac{T}{m}$. By \eqref{asmp_fLip}, it holds that
\begin{align}
    |I_4|\le\int^t_{0}|f(\bar{x}(\tau),m\tau)-f(x_m(\tau),m\tau)|d\tau\le \Lip (f)\int^t_{0}|\bar{x}(\tau)-x_m(\tau)|d\tau.\label{eq10}
\end{align}
Hence, by (\ref{eq7}), (\ref{eq8}), (\ref{eq9}), and (\ref{eq10}), we obtain for $m>m_\epsilon$,
\begin{align*}
    |\bar{x}(t)-x_m(t)|
    &\le (2\Lip (f)T+1)\epsilon+\Lip (f)\int^t_{0}|\bar{x}(\tau)-x_m(\tau)|d\tau.
\end{align*}
By Gronwall's inequality, if $m>m_\epsilon$, then
\begin{align*}
    |\bar{x}(t)-x_m(t)|\le (2\Lip (f)T+1)\epsilon e^{\Lip (f)T}.
\end{align*}
Since $\epsilon$ is arbitrary, we obtain $x_m\rightrightarrows \bar{x}$ as $m\rightarrow\infty$.
\end{proof}

\section*{Acknowledgments}

This work was partially supported by JSPS KAKENHI Grant Number JP20KK0058.

\end{document}